\documentclass[10pt,twocolumn]{article}
\usepackage[T1]{fontenc}
\usepackage[utf8]{inputenc}
\usepackage[english]{babel}
\usepackage[super,sort,compress]{cite}
\usepackage{algorithm}
\usepackage{algorithmic}
\usepackage{url}
\usepackage{verbatim}
\usepackage{graphicx}
\usepackage{booktabs}
\usepackage{multirow}
\usepackage{amsmath}
\usepackage{amssymb}
\usepackage{amsthm}
\usepackage{xcolor}
\usepackage[a4paper,margin=0.72in]{geometry}
\usepackage{microtype}
\usepackage[hidelinks]{hyperref}
\usepackage{authblk}
\hypersetup{
  pdftitle={Noisy group neurons with synchronous resetting for high-performance spiking neural networks},
  pdfauthor={Yajie Zhai, Yanmei Kang, Meng Li, and Zigang Huang},
  pdfkeywords={spiking neural network, stochastic resonance, gradient mismatching, synchronous resetting, mean-field learning}
}
\newtheorem{theorem}{Theorem}
\newtheorem{proposition}{Proposition}
\newcommand{\keywords}[1]{\par\noindent\textbf{Keywords:} #1\par}
\makeatletter
\newenvironment{algorithmhere}[1][]{%
  \par\smallskip\noindent\begin{minipage}{\columnwidth}\hrule height 0.6pt\kern 3pt%
  \refstepcounter{algorithm}%
  \def\caption##1{\noindent\textbf{Algorithm \thealgorithm.} ##1\par\kern 3pt\hrule height 0.4pt\kern 3pt}%
}{\kern 3pt\hrule height 0.6pt\end{minipage}\par\smallskip}
\newenvironment{figurehere}{%
  \par\begingroup\parskip0pt\topskip 6pt plus 1pt minus 1pt\def\@captype{figure}%
}{%
  \par\endgroup\vskip 6pt plus 1pt minus 1pt%
}
\newenvironment{tablehere}{%
  \parskip0pt\topskip 6pt plus 1pt minus 1pt\def\@captype{table}\begin{small}\begin{center}%
}{%
  \end{center}\end{small}\vskip 6pt plus 1pt minus 1pt%
}
\makeatother

\begin{document}
\title{Noisy group neurons with synchronous resetting for high-performance spiking neural networks}
\author[1]{Yajie Zhai}
\author[1,2]{Yanmei Kang\thanks{Corresponding author: ymkang@xjtu.edu.cn}}
\author[3,4]{Meng Li}
\author[3,4]{Zigang Huang\thanks{Corresponding author: huangzg@xjtu.edu.cn}}
\affil[1]{Department of Applied Mathematics, School of Mathematics and Statistics, Xi'an Jiaotong University, Xi'an 710049, China}
\affil[2]{Center for Intersection of Mathematics and Life Sciences, Xi'an Jiaotong University, Xi'an 710049, China}
\affil[3]{The Key Laboratory of Biomedical Information Engineering of Ministry of Education, Institute of Health and Rehabilitation Science, School of Life Science and Technology, Xi'an Jiaotong University, Xi'an 710049, China}
\affil[4]{Research Center for Brain-inspired Intelligence, Xi'an Jiaotong University, Xi'an 710049, China}
\date{}
\makeatletter
\twocolumn[
\begin{@twocolumnfalse}
\maketitle
\begin{abstract}
Spiking neural networks (SNNs), characterized by bio-inspired neuronal dynamics and event-driven communication, have attained significant progress in recent years. Nevertheless, training deep SNNs remains challenging due to spatiotemporal information loss and gradient mismatching. To simultaneously address these issues, we propose a noisy group neuron (NGN) model, which incorporates population-level synchronous resetting and neural stochasticity as fundamental computational mechanisms. We then develop the NGN method as a framework that combines the NGN model with backpropagation learning based on mean-field dynamics. We demonstrate the advantages of the NGN method through theoretical analysis and experimental validation on CIFAR-10, CIFAR-100, Tiny-ImageNet, DVS-Gesture, N-Caltech101, and CIFAR10-DVS. The proposed approach achieves an accuracy of 87.35\% on CIFAR10-DVS within 10 inference time steps. These results support NGN as a practical approach to high-performance neuromorphic computing.
\end{abstract}
\vspace{1em}
\end{@twocolumnfalse}
]
\makeatother
\section{Introduction}

Spiking Neural Networks (SNNs) are recognized as the third generation of neural networks for their biological plausibility and event-driven sparse communication~\cite{ref1,ref2,refIJNSreview}. Inspired by human brains, spiking neurons, as the building blocks of SNNs, emulate neural information processing via the evolution of the membrane potential and spiking mechanism. The temporal dynamics make SNNs particularly advantageous for handling asynchronous, event-driven data streams from Address-Event Representation (AER) neuromorphic sensors, such as Dynamic Vision Sensors (DVS)~\cite{ref4}. Consequently, SNN-based algorithms have proven effective in applications such as vision sensing~\cite{ref5}, image recognition and segmentation, odor classification, policy optimization, multimodal inference, fault diagnosis~\cite{ref_eswa_iipooling}, and autonomous robotics~\cite{refMurase}. With the development of neuromorphic hardware, such as IBM's TrueNorth~\cite{ref11} and Intel's Loihi 2 chips~\cite{ref12}, SNNs have already become promising candidates for edge computing and low-power AI.

The existing SNNs primarily employ leaky integrate-and-fire (LIF) neurons to balance biological plausibility with computational simplicity~\cite{refIJNSlif}. This has enabled significant advances in neuromorphic computing, but it introduces two fundamental challenges that limit the training efficiency and performance of SNNs. The first challenge arises from information loss in spatiotemporal encoding by LIF neurons~\cite{ref19}. Due to the binary spike representation and single-channel connection architecture, LIF neurons face a representational dilemma in simultaneously encoding spatial intensity distributions and temporal dynamics~\cite{ref_fanMSF}. This limitation becomes particularly evident over brief temporal windows: when facing subthreshold stimuli of varying strengths, the absence of spike responses renders different intensity levels indistinguishable, despite producing distinct membrane potential waveforms. The second challenge stems from the gradient mismatching. The LIF neuron operates following a non-differentiable firing-and-resetting mechanism~\cite{ref13}, which prevents direct backpropagation in gradient-based supervised learning~\cite{ref14}. To overcome this issue, surrogate gradient learning (SGL) methods utilize smooth approximations of the derivative associated with spike generation~\cite{ref15}. Despite their significant success in training deep SNNs~\cite{ref15,ref16}, SGL methods inherently suffer from a critical inconsistency between the continuous gradients in backpropagation and the discrete binary spikes generated during forward inference. This forward--backward discrepancy creates gradient mismatching and reduces learning efficiency~\cite{ref17,ref18}.

Both challenges are interconnected manifestations of the same underlying structural constraints imposed by the all-or-none nature of spike generation in LIF neurons, and various overcoming techniques have been developed from two aspects. On one hand, efforts including adaptive smoothing mechanisms~\cite{ref18}, finite difference gradients~\cite{ref14}, surrogate module learning~\cite{ref47}, and complementary membrane potential dynamics~\cite{ref51} have achieved notable success in speeding up the convergence of training and alleviating the gradient mismatching. On the other hand, diverse neuronal models have been explored for enriching the expressive power, from learnable LIF neurons~\cite{ref22,ref50} and dynamically regulated neurons~\cite{ref_eswa_tf} to multi-compartment neurons~\cite{ref24} and parallel-neuron architectures~\cite{MLF2022,PSN2023}, so that information loss can be minimized. Along this direction, Adaptive Fission allocates weighted neuron groups to sensitive units for low-latency population coding~\cite{ref_adaptive_fission}, whereas multiplication-free parallelizable spiking neurons pursue efficient spatiotemporal dynamics~\cite{ref_mulfree_psn}. Nevertheless, although the methods in the first category enhance training convergence and gradient flow, they do not fundamentally address the information loss inherent in spatiotemporal encoding. The architectural modifications in the second category enrich expressive power, but they tend to introduce substantial optimization difficulties. These techniques therefore result in a problematic trade-off: improvements in one dimension come at the expense of the other, leaving most existing SNNs unable to achieve both efficient training and high-performance inference simultaneously. Therefore, a valuable trade-off between training efficiency and encoding accuracy remains necessary from fundamentally different architectural principles.

To seek such a trade-off, the proposed model starts from the collective response of a finite noisy-neuron array. Each member receives the same deterministic presynaptic input and an independent noise sample. Averaging their binary responses produces a graded spike with resolution $1/K$, while constructive noise can expose weak subthreshold differences through stochastic-resonance effects~\cite{ref23,ref25,ref26,ref27,ref28,ref29,ref39}. We further introduce population-level synchronous resetting. After the group response is aggregated, all members start the next time step from the same aggregated post-spike state. This operation couples the otherwise independent members and avoids tracking $K$ independent membrane histories. The response-dependent reduction of the subsequent shared state creates competitive suppression, a functional effect that has also been described as lateral inhibition in related work~\cite{ref35}.

Based on the above two mechanisms, we propose a noisy group neuron (NGN) to address spatiotemporal information loss and gradient mismatching simultaneously. The finite group response is an empirical firing probability, while the Gaussian surrogate used in backpropagation is the corresponding population probability. This correspondence enables us to quantify how the forward response approaches the backward matching signal as $K$ increases. The paper is structured as follows. Section 2 presents the NGN dynamics and synchronous-resetting implementation. Section 3 analyzes neuron-level stochastic resonance and information enhancement. Section 4 develops the mean-field learning algorithm and studies gradient mismatching. Section 5 reports classification, ablation, group-size, and computational-cost experiments. Section 6 gives the conclusion.

\section{Noisy group neuron model}
\subsection{Population dynamics and mean-field approximation}
Following the statistical-neurodynamics perspective~\cite{refAmari}, our group neuron model is developed based on an array of $K$ noisy identical independent LIF neurons. Let $V_k(t)$ denote the membrane potential of the $k$-th member neuron for $1 \leq k \leq K$. Then, the evolution of the membrane potential can be delineated as
\begin{equation}
	\label{eq1}
	\begin{aligned}
		\tau_m\,dV_k(t)
		&=\left[-V_k(t)+R_m I_0\right]dt+\sigma_0\,dW_k(t),\\
		&\hspace{3.1cm} V_k(t) \leq V_{th}.
	\end{aligned}
\end{equation}
where $\tau_m$ is the membrane time constant, $R_m$ is the membrane resistance constant, $V_{th}$ is a prescribed spike threshold, $I_0$ is the deterministic presynaptic current input, $dW_k(t)$ is the independent increment of Wiener process to describe the stochastic fluctuation, and $\sigma_0$ is usually referred to as the noise intensity. 

Once $V_k(t)$ reaches the threshold $V_{th}$ from below, a spike (namely, action potential) is emitted, the membrane potential is immediately returned to a resting potential $V_{re}$, and the evolution in Eq.~(\ref{eq1}) is restarted. Assuming that the $k$th neuron remains in the sub-threshold regime from $t$ to $t+\Delta t$, we obtain
\begin{equation}  
	\label{eq2}  
	\begin{aligned}  
		V_k(t+\Delta t) &= \left(1 - e^{-\frac{\Delta t}{\tau_m}}\right) R_m I_0 + V_k(t) e^{-\frac{\Delta t}{\tau_m}} \\
		&\quad + \frac{\sigma_0}{\tau_m} \int_t^{t+\Delta t} e^{-\frac{t+\Delta t - s}{\tau_m}} dW_k(s)  
	\end{aligned}  
\end{equation}
is a Gaussian process of zero mean. By Ito's isometry,
\begin{equation}
	\label{eq3}
	E\left(\frac{\sigma_0}{\tau_m} \int_t^{t+\Delta t} e^{\frac{-(t+\Delta t-s)}{\tau_m}} dW_k(s)\right)^2=\frac{\sigma_0^2}{2 \tau_m}\left(1-e^{-\frac{2 \Delta t}{\tau_m}}\right).
\end{equation}
If $\xi_k(t)=\frac{\sigma_0}{\tau_m} \int_t^{t+\Delta t} e^{-\frac{t+\Delta t-s}{\tau_m}} dW_k(s)$, then 
\begin{equation}
	\xi_k(t) \sim N\left(0,\frac{\sigma_0^2}{2\tau_m}
	\left(1-e^{-\frac{2\Delta t}{\tau_m}}\right)\right).
\end{equation}
We fix $R_m=1 /(1-\tau)$ with $\tau=e^{-\frac{\Delta t}{\tau_m}}$ and denote $V_k(t+\Delta t)$ as $V_k^{t+1}$ for simplicity of notation. Then, by Eq.~(\ref{eq2}), we have
\begin{equation}
	\label{eq4}
	V_k^{t+1}=\tau V_k^t+I_0+\sigma \eta_k^t,
\end{equation}
where $\sigma=\sigma_0 \sqrt{\frac{1}{2\tau_m}\left(1-e^{-\frac{2 \Delta t}{\tau_m}}\right)}$ and $\eta_k^t \sim N(0,1)$. Here, we take $\Delta t$ as a hyperparameter such that $\Delta t<\tau_m$ \cite{ref36} and the $\Delta t$-dependent noise intensity $\sigma$ can ensure that the discrete model statistically resembles its continuous counterpart. 

Let $O_k^t$ represent the spike of the $k$th neuron at time $t$, such that $O_k^t=\Theta\left(V_k^t-V_{th}\right)$. Incorporating the hard-resetting mechanism into Eq.~(\ref{eq4}) yields
\begin{equation}
	\label{eq5}
	V_k^{t+1}=\tau H_k^t+I_0+\sigma \eta_k^t, \quad 1 \leq k \leq K,
\end{equation}
where $H_k^t=V_k^t\left(1-O_k^t\right)+V_{re} \delta_{1, O_k^t}$ and $\delta_{1, O_k^{\prime}}$ is Dirac notation. However, it is computationally expensive to track the dynamics of each sub-neuron in Eq.~(\ref{eq5}). We therefore model the collective behavior through a mean-field approximation. Averaging both sides of Eq.~(\ref{eq2}) over $k=1,2, \ldots, K$ gives
\begin{equation}
	\label{eq6}
	\begin{aligned}  
		\bar{V}(t+\Delta t)=&\left(1-e^{-\frac{\Delta t}{\tau_m}}\right) R_m I_0+\bar{V}(t) e^{-\frac{\Delta t}{\tau_m}}+\\
		&\frac{\sigma_0}{\tau_m} \int_t^{t+\Delta t} e^{\frac{-(t+\Delta t-s)}{\tau_m}} \varsigma(s) d s,
	\end{aligned}
\end{equation}
where $\varsigma(s)=\frac{1}{K} \sum_{k=1}^K dW_k(s)$ and $\bar{V}(t)$ is the average membrane potential.

By the law of large numbers, the arithmetic average $\varsigma(s)\xrightarrow{p}0$ as $K\rightarrow\infty$. Therefore, taking the limit $K\rightarrow\infty$ in Eq. (\ref{eq6}), we obtain the mean-field dynamics
\begin{equation}
	\label{eq7}
	\bar{V}(t+\Delta t)=\left(1-e^{-\frac{\Delta t}{\tau_m}}\right) R_m I_0+\bar{V}(t) e^{-\frac{\Delta t}{\tau_m}},
\end{equation}
in the large $K$ limit.

With the hard-threshold resetting mechanism and the simplified notation in Eq. (\ref{eq5}) taken into account, Eq. (\ref{eq7}) can be simplified into
\begin{equation}
	\label{eq9}
	\bar{V}^{t+1}=\tau\lim_{K\rightarrow\infty}\bar{H}_{K}^t+I_0,
\end{equation}
where
\begin{equation*}
	\bar{H}_{K}^t=\frac{1}{K}\sum_{k=1}^K V_k^t\left(1-O_k^t\right)
	+\frac{V_{re}}{K}\sum_{k=1}^K\delta_{1,O_k^t}
\end{equation*}
is the average resetting voltage across all subneurons. To further simplify Eq.~(\ref{eq9}), we employ a mean-field approximation~\cite{ref37}. When $V_{re}=0$, Eq.~(\ref{eq9}) can be approximated as
\begin{equation}
	\label{eq11}
	\bar{V}^{t+1}=\tau\bar{V}^t\left(1-\bar{O}^t\right)+I_0,
\end{equation}
where $\bar{O}^t$ is the average firing rate of subneurons and $\bar{V}^t$ is the average membrane potential at time $t$. Here the approximation relies on the mean-field assumption that the membrane potentials of subneurons $V_k^t\sim\mathcal{N}(\bar{V}^t,\sigma^2)$. Then, in the large $K$ limit, it is obvious that 
\begin{equation}
\begin{aligned}
	\lim_{K\to\infty}\bar{H}_K^t
	={}&\bar{V}^t\left(1-\bar{O}^t\right)
	-\frac{\sigma}{\sqrt{2\pi}}
	\exp\left(-\frac{(V_{th}-\bar{V})^2}{2\sigma^2}\right)\\
	&+V_{re}\bar{O}^t.
\end{aligned}
\end{equation}
Notice that $\frac{\sigma}{\sqrt{2 \pi}} \exp \left(-\frac{\left(V_{th}-\bar{V}\right)^2}{2 \sigma^2}\right) \approx 0$ when $\sigma<1$. Then when $V_{re}=0$, we get the mean-field equation~(\ref{eq11}) from Eq.~(\ref{eq9}).
\subsection{Finite-size sampling implementation}
The mean-field approximation in Section 2.1 characterizes the idealized limit $K\rightarrow\infty$. In implementation, we use a finite population of $K$ members to trade sampling accuracy against computation and memory.

With the above preparation, we formulate the finite-size sampling version of the NGN model. Let $v_k^{t,n}$ and $o_k^{t,n}$ denote the voltage and spike output of member $k$ on layer $n$ at time $t$. The noisy presynaptic input is $\tilde{I}_k^{t,n-1}=I^{t,n-1}+\sigma\eta_k^{t,n-1}$, where $I^{t,n-1}$ is shared by the group and $\eta_k^{t,n-1}\sim N(0,1)$ is sampled independently for each member.

Similarly to Eq.~(\ref{eq5}), when $V_{re}=0$, the voltage and spike of each member neuron are calculated as
\begin{equation}
	\label{eq12}
	v_k^{t, n}=\tau h_k^{t-1, n}+\tilde{I}_k^{t, n-1}, o_k^{t, n}=\Theta\left(v_k^{t, n}-V_{th}\right),
\end{equation}
where $h_k^{t-1, n}$ is the resetting membrane potential of member neuron $k$. The averaged membrane potential $\bar{v}^{t, n}$ and aggregated firing rate $\bar{o}^{t, n}$ within the $K$-member NGN model are computed as
\begin{equation}
	\label{eq13}
	\bar{v}^{t, n}=\tau \bar{h}^{t-1, n}+I^{t, n-1}, \quad \bar{o}^{t, n}=\frac{1}{K} \sum_{k=1}^K o_k^{t, n}
\end{equation}
with the synchronous-resetting state $\bar{h}^{t,n}=\bar{v}^{t,n}(1-\bar{o}^{t,n})$. Here, $\bar{o}^{t,n}$ is a graded spike taking values in $\{0,\frac{1}{K},\ldots,1\}$ rather than a binary spike. Such graded event representations are supported by neuromorphic hardware including Intel Loihi 2~\cite{ref12}.

To couple the population through synchronous resetting and reduce state-tracking complexity, we modify Eq.~(\ref{eq12}) as
\begin{equation}
	\label{eq14}
	v_k^{t, n}=\tau \bar{h}^{t-1, n}+\tilde{I}_k^{t, n-1}, o_k^{t, n}=\Theta\left(v_k^{t, n}-V_{th}\right).
\end{equation}
Replacing each $h_k^{t-1,n}$ with $\bar{h}^{t-1,n}$ synchronously initializes all members from one post-spike state. A larger $\bar{o}^{t,n}$ more strongly reduces the next shared state, producing competitive suppression without an explicit inhibitory circuit; related suppression has also been termed lateral inhibition~\cite{ref35}. Eqs.~(\ref{eq13})--(\ref{eq14}) and Fig.~\ref{fig_1} define the resulting graded response and compact recurrence. NGN reduces to stochastic LIF at $K=1$ and to deterministic LIF as $\sigma\to0$.
\section{Effects of aperiodic stochastic resonance}
We use mutual information to quantify stochastic-resonance effects in NGN and relate neuron-level encoding to network performance~\cite{ref19,ref38,ref39}.

Let $M I\left(I^t, o^t\right)$ be the mutual information between the presynaptic input $I^t$ and the output spike $o^t$ at time $t$, then we have
\begin{equation}
	\label{eq15}
	M I\left(I^t, o^t\right)=H\left(o^t\right)-H\left(o^t \mid I^t\right),
\end{equation}
where $H(o^t)=-\sum_o p(o^t)\log_2p(o^t)$ and $H(o^t\mid I^t)=\int f_{I^t}(y)H(o^t\mid I^t=y)dy$. Thus, $MI(I^t,o^t)$ measures how much input information is retained by the neuronal response.

\begin{figurehere}
	\centering
	\includegraphics[width=\columnwidth]{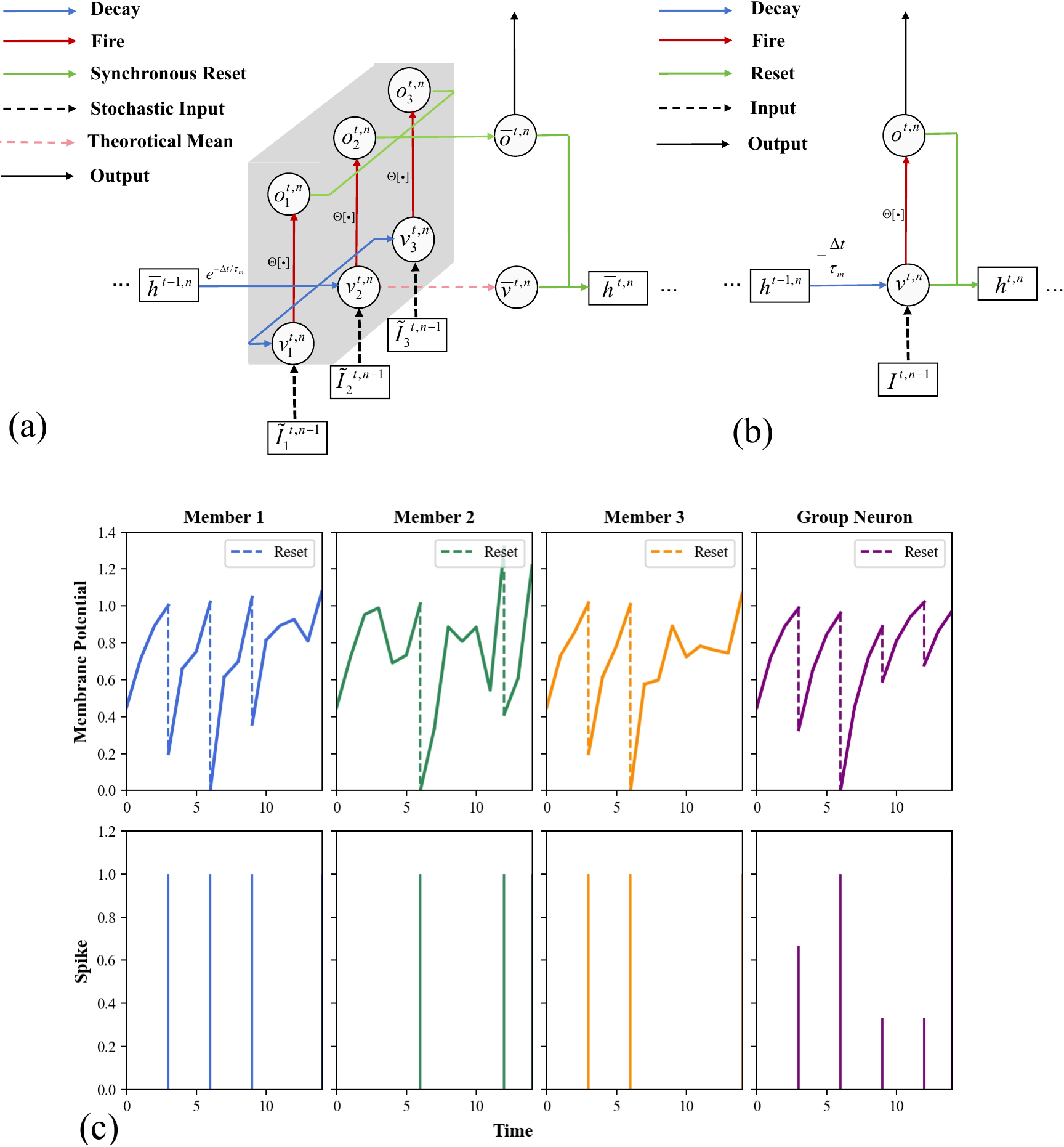}
	\caption{NGN architecture and temporal response for $K=3$. (a) Members integrate independently perturbed inputs from a shared post-resetting state; their averaged spikes determine the next shared state. (b) Conventional LIF dynamics. (c) Example member and group responses for $\Delta t=0.5$, $I^t=0.4$, $\tau_m=1$, $V_{th}=1$, and $V_{re}=0$; dashed segments denote resetting.}
	\label{fig_1}
\end{figurehere}

For NGN models, $I^{t, n-1}$ is the presynaptic input from the $(n-1)$ th layer at time $t, \bar{o}^{t, n}$ is the output spike state and $\bar{v}^{t, n}$ is the spatially accumulated membrane potential of the $n$th layer at time $t$. Then it can be deduced that $M I\left(I^{t, n-1}, \bar{o}^{t, n}\right)=M I\left(\bar{v}^{t, n}, \bar{o}^{t, n}\right)$ by the following Proposition \ref{pro1}.

\begin{proposition}\label{pro1}
	Assume that the presynaptic inputs $I^{t,n-1}= I \sim N(\mu_I, \sigma_I^2)$ for all $t \geq 1$. Given the observed spike values 
	$\bar{o}^{1,n}, \bar{o}^{2,n}, \ldots, \bar{o}^{t-1,n}$ of one NGN model, 
	the membrane potential $\bar{v}^{t,n}$ approximately follows 
	$N(\mu_v, \sigma_v^2)$ where $\sigma_v^2 \propto \sigma_I^2, \mu_v \propto \mu_I$.
\end{proposition} 
\begin{proof}
\end{proof}

Proposition \ref{pro1} demonstrates that the aggregated membrane potential has the same distribution type as the presynaptic input when the presynaptic input from the previous layer is consistent in time. This assumption and the property can be seen as an extension of the deterministic LIF neuron with Gaussian input \cite{ref40}. Nevertheless, the spatially aggregated spike $\bar{o}^{t,n}$ follows a different distribution, as stated in Proposition 2.
\begin{proposition}\label{pro2}
	For the given membrane potential $\overline{v}^{t, n}$, we have $K \bar{o}^{t, n} \sim B\left(K, p^{t, n}\right)$. That is, the probability mass function for $\bar{o}^{t, n}$ has the following form
	\begin{equation}
		\label{eq16}
		\begin{aligned}
			P\left\{\bar{o}^{t, n}=\frac{k}{K}\right\}&\triangleq f_{\bar{o}^{t, n}}\left(\left.\frac{k}{K} \right\rvert\,\bar{v}^{t, n}\right)\\
			&=\binom{K}{k}\left(p^{t, n}\right)^k\left(1-p^{t, n}\right)^{K-k}
		\end{aligned}
	\end{equation}
where $K$ denotes the group size of the NGN model and $p^{t, n}$ is the probability that a firing event occurs at time $t$, given by
	\begin{equation}
		\label{eq17}
		p^{t, n} \triangleq p_k^{t, n}
		=1-F_\eta\left(\frac{V_{th}-\bar{v}^{t, n}}{\sigma_\eta}\right).
	\end{equation}
	Here $F_\eta$ is the cumulative density function of the injected Gaussian white noise $\eta$ of intensity $\sigma_\eta$.	
\end{proposition} 

Building upon Propositions \ref{pro1} and \ref{pro2}, we obtain the analytical form of mutual information for the NGN model in Theorem \ref{the1}.
\begin{theorem}\label{the1}
	Assume that the presynaptic input $I^{t, n-1}\sim N(\mu_I,\sigma_I^2)$ are i.i.d. for all $t \geq 1$. The mutual information $MI(I^{t, n}, \bar{o}^{t, n})$  between $I^{t, n-1}$ and graded spike $\bar{o}^{t,n}$ in the NGN model is
	\begin{equation}\label{eq18}  
		\begin{aligned}  
			MI(I^{t, n-1}, \bar{o}^{t, n})	 &=   
			- \sum_{k=0}^K f_{\bar{o}^{t, n}}\left(\frac{k}{K}\right)   
			\log_2\left(\frac{f_{\bar{o}^{t, n}}\left(\frac{k}{K}\right)}{\binom{K}{k}}\right) \\
			& + K \int_{-\infty}^{\infty} f_{\bar{v}^{t, n}}(v)   
			p^{t, n} \log_2 p^{t, n} \, dv \\
			&+ K \int_{-\infty}^{\infty} f_{\bar{v}^{t, n}}(v)   
			\left(1 - p^{t, n}\right) \\
			&\quad \times \log_2\left(1 - p^{t, n}\right)\,dv,  
		\end{aligned}  
	\end{equation}
	where $f_{\bar{o}^{t, n}}\left(\frac{k}{K}\right)$ in the first term is
	\begin{equation*}  
		\begin{aligned}  
			f_{\bar{o}^{t, n}}\left(\frac{k}{K}\right) = &  
			\frac{\sigma_\eta}{\sqrt{\pi} \sigma_v}   
			\binom{K}{k}   
			\int_{-\infty}^{\infty}h(u)^k \left(1 - h(u)\right)^{K-k}\\  
			&\exp\left(  
			-\left(  
			\frac{\sigma_\eta u}{\sigma_v}   
			+ \frac{V_{\text{th}} - \mu_v}{\sqrt{2} \sigma_v}  
			\right)^2  
			\right)   
			\, du,  
		\end{aligned}  
	\end{equation*} 
	with $h(u)=\frac{1}{2} \operatorname{erfc}(-u)$ and $f_{\bar{v}^{t n}}(v)$ in the second term reads
	\begin{equation*}
		f_{\bar{v}^{t,n}}(v)=\frac{1}{\sqrt{2 \pi} \sigma_v} \exp \left(-\frac{\left(v-\mu_v\right)^2}{2 \sigma_v^2}\right),
	\end{equation*}
	with $\sigma_v \propto \sigma_I$,$\mu_v \propto \mu_I$.
\end{theorem}
\begin{proof}
\end{proof}

\begin{table*}[htbp!]
	\centering
	\caption{The relationship between neuronal mutual information and network performance.}
	\begin{tabular}{llccccc}
		\hline
		\textbf{Neuron} & \textbf{Metric} & $\sigma=0$ & $\sigma=0.25$ & $\sigma=0.5$& $\sigma=0.75$ & $\sigma=1$ \\
		\hline
		\multirow{2}{*}{Net neuron 1} & MI & 0.30 & 0.43 & \textbf{0.47} & 0.44 & 0.37 \\
		& Acc(\%) & 91.19 & 93.22 & \textbf{94.50} & 70.27 & 9.23 \\
		\hline
		\multirow{2}{*}{Block1.neuron1} & MI & 0.21 & 0.32 & \textbf{0.35} & 0.33 & 0.28 \\
		& Acc(\%) & 92.37 & 93.60 & \textbf{94.50} & 76.20 & 9.09 \\
		\hline
		\multirow{2}{*}{Block1.neuron2} & MI & 0.41 & 0.54 & \textbf{0.55} & 0.49 & 0.41 \\
		& Acc(\%) & 92.12 & 93.82 & \textbf{94.50} & 79.75 & 10.12 \\
		\hline
		\multirow{2}{*}{Block2.neuron1} & MI & 0.20 & 0.30 &\textbf{0.34} & 0.32 & 0.27 \\
		& Acc(\%) & 91.08 & 93.22 & \textbf{94.50} & 83.52 & 14.13 \\
		\hline
		\multirow{2}{*}{Block2.neuron2} & MI & 0.34 & 0.47 & \textbf{0.48} & 0.45 & 0.38 \\
		& Acc(\%) & 91.58 & 93.35 & \textbf{94.50} & 85.86 & 13.99 \\
		\hline
	\end{tabular}
	\label{tab0}
\end{table*}

To validate Theorem \ref{the1}, we have recorded both Gaussian and non-Gaussian presynaptic inputs $I^{t, n-1}$ from a well-trained ResNet-18 during the validation stage for classification of CIFAR-10. Then the mutual information $MI\left(I^{t, n-1}, \bar{o}^{t, n}\right)$ is computed to characterize the stochastic resonance behavior of the NGN model under different inputs. We first consider the Gaussian presynaptic inputs (Fig.~\ref{fig2}(a)), where the Q-Q plot confirms the Gaussian nature of the distribution. As shown in Fig.~\ref{fig2}(b)--(c), the theoretical predictions (solid curves) agree well with the simulation results (dashed curves). Notably, mutual information exhibits a non-monotonic dependence on noise intensity when $K \geq 2$, indicating the occurrence of stochastic resonance. In contrast, this phenomenon is absent in the single neuron model ($K = 1$). This demonstrates a key advantage of the nontrivial NGN model. 

For non-Gaussian inputs that cannot be described by Theorem \ref{the1}, we recorded presynaptic inputs from the trained ResNet-18. The Q-Q plot in Fig.~\ref{fig2}(d) clearly reveals non-Gaussianity. We compute the mutual information through direct numerical simulation, and as shown in Fig.~\ref{fig2}(e)--(f), stochastic resonance still emerges only when $K \geq 2$. This confirms that noise can enhance information transmission in the NGN model more effectively than in the single LIF model, regardless of the input distribution. Moreover, by comparing Fig.~\ref{fig2}(c) and Fig.~\ref{fig2}(f), we observe that the optimal noise intensity remains relatively consistent for both Gaussian inputs and for non-Gaussian inputs across the time domain. This consistency across different input distributions and iterations suggests that noise intensity can be reasonably fixed as a hyperparameter in our algorithm design.

Two remarks are pertinent. First, stochastic resonance tends to occur only for nontrivial groups ($K > 1$), resembling suprathreshold stochastic resonance~\cite{ref28}. Unlike a parallel array whose members retain independent histories~\cite{ref41,ref42}, synchronous resetting makes the next effective baseline depend on the current aggregated response. Second, parameter-induced stochastic resonance can be obtained by tuning model parameters rather than noise intensity~\cite{ref23,ref26,ref28}. Because connection weights are learned while noise intensity is treated as a hyperparameter, NGN can be interpreted through this parameter-induced suprathreshold-resonance perspective.

We further demonstrate that the enhancement of neuronal information encoding capability directly translates to improved network performance through a controlled experimental design. We maintain a constant noise intensity for all layers preceding the target neuron while systematically varying the noise intensity $\sigma$ for the target neuron and subsequent layers. Table~\ref{tab0} presents the mutual information and classification accuracy of ResNet-18 under different noise intensities, measured at various neuronal layers. 

The results show a consistent trend across the examined neurons, from the input layer (Net neuron 1) to deeper residual blocks (Block 1 and Block 2). Both mutual information and network accuracy exhibit a single-peak dependence on noise intensity: they increase with noise intensity, reach a maximum at $\sigma=0.5$, and then decrease at higher noise levels. This pattern is consistent with stochastic resonance in the NGN model. In particular, the noise intensity $\sigma=0.5$ that maximizes mutual information also yields the best classification accuracy (94.50\%). The agreement between neuronal-level information transmission and network-level performance suggests that enhanced information encoding is associated with improved classification performance.

\begin{figure*}[!t]
	\centering
	\includegraphics[width=5.7in]{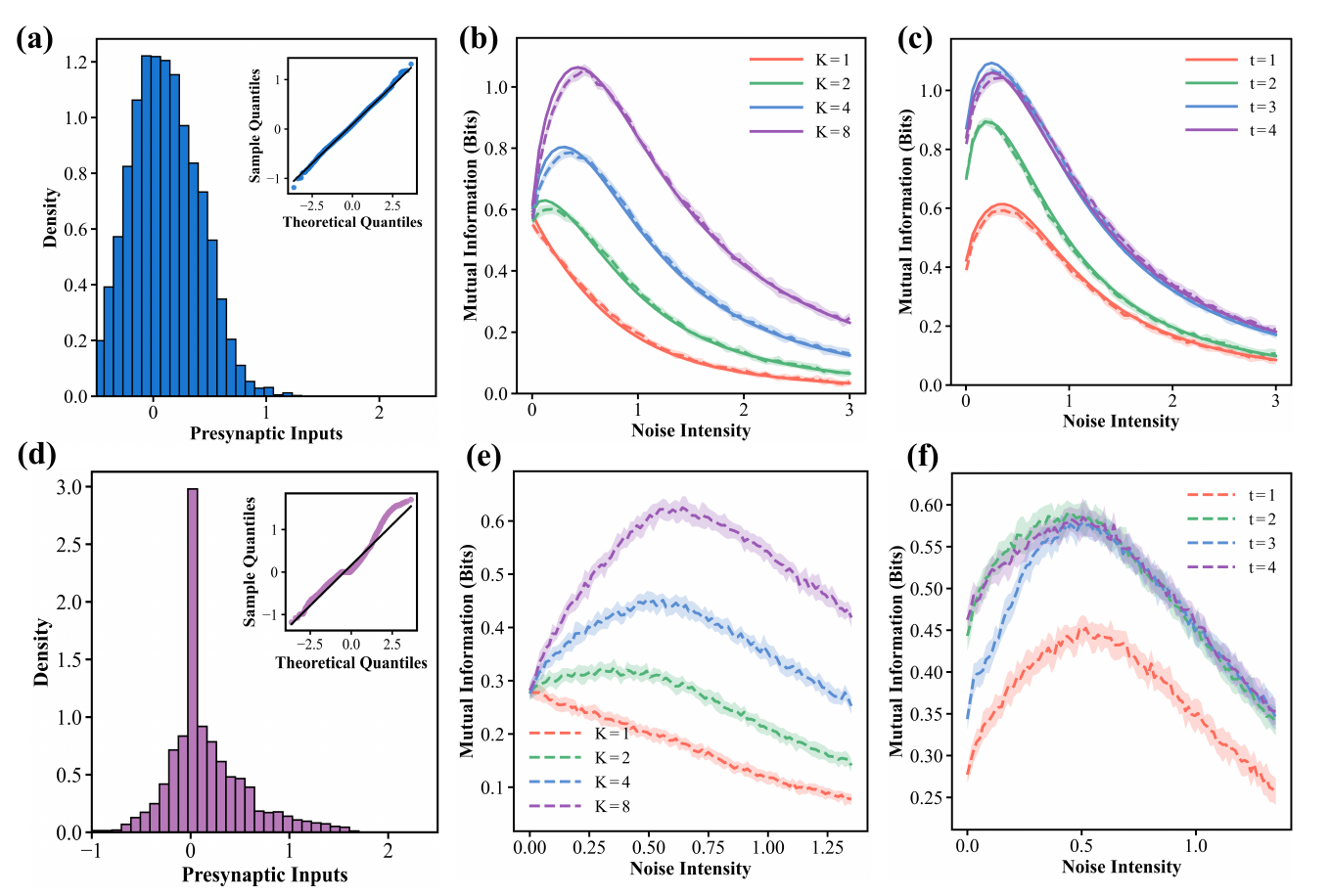}
	\caption{Stochastic resonance behavior of the NGN model characterized by the input-output mutual information $M I\left(I^{t, n-1}, \bar{O}^{t, n}\right)$ under the recorded presynaptic Gaussian inputs and non-Gaussian inputs. The Gaussian presynaptic input (a) and the corresponding input-output mutual information versus noise intensity are shown in the first row under different group sizes or iteration: (b) $t=1$; (c) $K=4$. Here, the theoretical prediction is shown in solid curves and the simulated results are shown in dashed curves. The shaded regions represent one standard deviation of the data. The non-Gaussian input (d) and the corresponding mutual information are shown in the second row under different group sizes or iteration steps: (e) $t=1$; (f) $K=4$. The initial noise intensity for training ResNet-18 is $\sigma_0=0.5$. The quantile-quantile (Q-Q) plot embedded in (a) and (d) signifies how the presynaptic input deviates from Gaussian distribution. The non-monotonic curves in (b), (c), (e) and (f) manifest the occurrence of stochastic resonance.}
	\label{fig2}
\end{figure*}

\begin{figure*}[!t]
	\centering
	\includegraphics[width=6.0in]{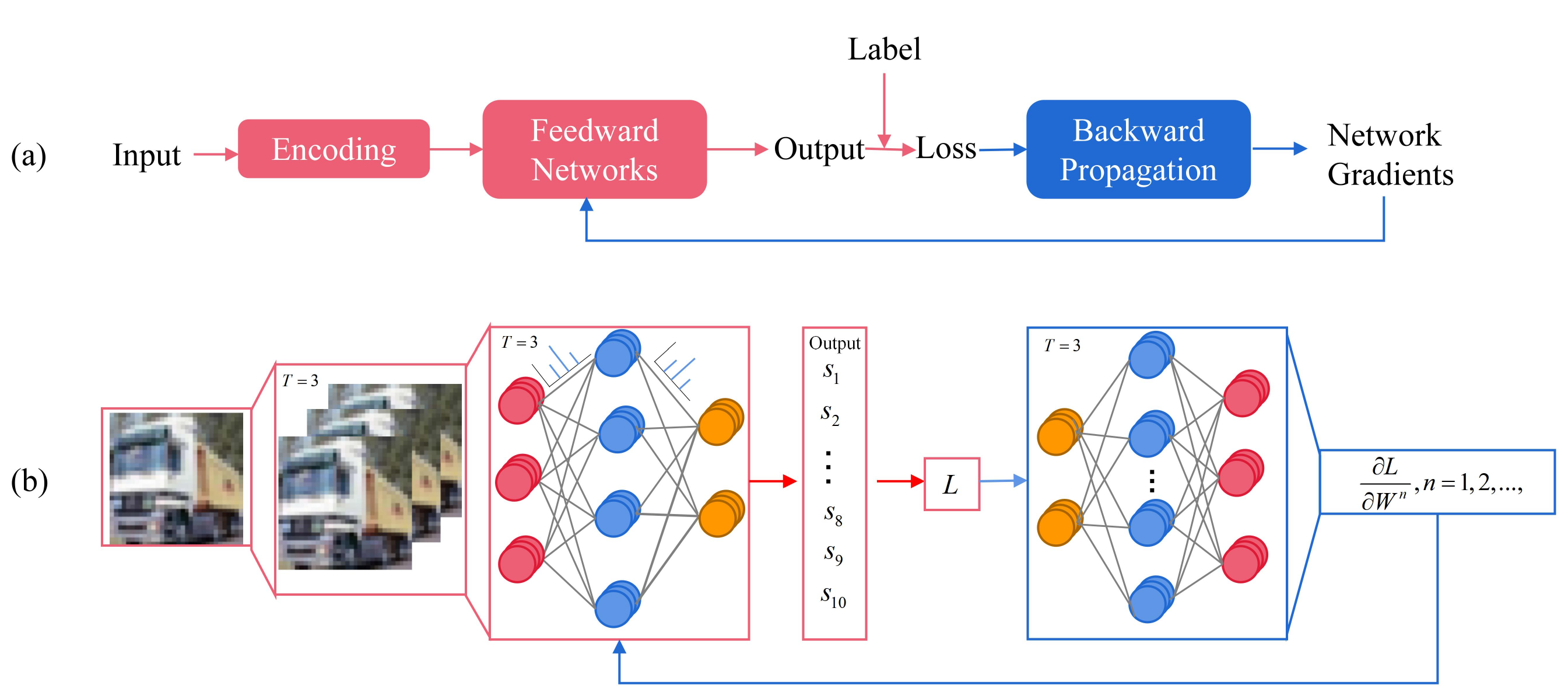}
	\caption{Flowcharts of the fully connected $N$-layer SNNs: (a) a general procedure and (b) the NGN-based procedure. Forward spike propagation and backward loss propagation are symmetric in gradient learning. Here, $\frac{\partial L}{\partial w^n}$ represents the gradient of the loss $L$ with respect to the weights $w^n$ for $n = 1, 2, \ldots$.}
	\label{fig3}
\end{figure*}

\section{Mitigating gradient mismatching in learning process}

Just as in the backpropagation learning of SNNs, the non-differentiability of spiking activities is still an annoying obstacle in the NGN method. The representative resolutions are ANN-to-SNN conversion \cite{ref35} and surrogate gradients (SG) \cite{ref15,ref16}. In SG-based learning, the gradient of spiking impulse is replaced by an appropriate surrogate function so that the backpropagation learning becomes accessible. Note that the SG-based learning requires fewer training steps and accordingly reduces inference costs, but the resultant gradient mismatching may lead to suboptimal outcomes \cite{ref17,ref18}.

In this paper, we adopt the SG-based learning framework and train NGN-based SNNs following the spatial and temporal back-propagation (STBP) paradigm. Importantly, rather than arbitrarily selecting a surrogate function, we derive the Gaussian surrogate function directly from the mean-field approximation established in Eq. (\ref{eq11}). This principled derivation not only provides theoretical justification for the choice of surrogate function but also helps mitigate gradient mismatching through the ensemble structure of the NGN model. For completeness, we now present this derivation.

\subsection{STBP-based learning rules}
Let us illustrate the training algorithm with a fully connected $N$-layer feedforward SNN of NGN models, with architecture shown in Fig.~\ref{fig3}. Given $T > 0$, by taking the firing rate $$s=\frac{1}{T} \sum_{t=1}^T \bar{o}^{t, N}$$ as the output of the network, the loss function $L$ can be defined through the cross-entropy or temporal efficient training (TET) loss \cite{ref43} as $$L=\mathcal{L}\left(s, y\right),$$ where $y$ is the label vector of the desired categories. Since the gradient training is to minimize the training loss function by updating the connecting weights along the negative gradient direction, the forward propagation of spikes and the backward propagation of errors are symmetrical in the gradient-based training, as shown in Fig.~\ref{fig3}. Hence, in order to elucidate the backward propagation of loss, it is imperative to first examine the forward propagation mechanism in the SNN within NGN methods.

As seen from Eqs. (\ref{eq13}) and (\ref{eq14}), each member neuron of the NGN model at $n$th layer receives the same deterministic presynaptic inputs $I^{t, n}=W^n \bar{o}^{t, n-1}$ and evolves instantaneously regulated by the graded spike output $\bar{o}^{t, n}$ and the aggregated membrane potential of
$\bar{v}^{t, n}$ in the NGN model. Here, $\bar{v}^{t, n}$ and $\bar{o}^{t, n}$ represent the integration of firing behaviors of all member neurons. In other words, it is the collective response of NGN models that is designed to dominate the feedforward propagation of the SNN, thus it is reasonable to assume that all the $K$ member neurons share the same gradients $\frac{\partial \bar{o}^{t, n}}{\partial \bar{v}^{t, n}}$. And by STBP \cite{ref15}, the parameter optimization rule for the NGN-based SNN can be expressed as
\begingroup
\setlength{\abovedisplayskip}{7pt}
\setlength{\belowdisplayskip}{7pt}
\begin{equation}
	\label{eq19}
	W^n \leftarrow W^n-\gamma \frac{\partial L}{\partial W^n},
\end{equation}
\endgroup
where
\begin{equation}
	\label{eq20}
	\begin{aligned}
		\frac{\partial L}{\partial W^n}&=\frac{\partial L}{\partial s} \sum_t \frac{\partial s}{\partial \bar{v}^{t, n}} \frac{\partial \bar{v}^{t, n}}{\partial W^n}, \\
		\frac{\partial L}{\partial \bar{u}^{t, n}}&=\frac{\partial L}{\partial \bar{o}^{t, n}} \frac{\partial \bar{o}^{t, n}}{\partial \bar{v}^{t, n}}+\frac{\partial L}{\partial \bar{o}^{t+1, n}} \frac{\partial \bar{o}^{t+1, n}}{\partial \bar{v}^{t, n}},
	\end{aligned}
\end{equation}
and $\gamma$ stands for the learning rate.
Nevertheless, as shown in Fig.~\ref{fig4}(a), the discrete graded spike $\bar{o}^{t, n}$ is stochastic and not everywhere differentiable, hence how to calculate $\frac{\partial \bar{o}^{t, n}}{\partial \bar{v}^{t, n}}$, namely the derivative of the graded spike train with respect to aggregated membrane potential, is still a barrier. 

To resolve this, we utilize the mean-field dynamics in Eq. (\ref{eq11}) to replace the evolution of individual member neurons in Eq. (\ref{eq13}) and Eq. (\ref{eq14}), which naturally provides a differentiable surrogate. That is to say, 
\begin{equation}\frac{\partial \bar{o}^{t, n}}{\partial \bar{v}^{t, n}}=\left.\frac{\partial \bar{O}^{t}}{\partial \bar{V}^{t}}\right|_{\bar{O}^{t}=\bar{o}^{t,n}, \bar{V}^{t}=\bar{v}^{t,n}}. \end{equation}
where the average firing rate $\bar{O}^{t}$ is given as
\begin{equation}
	\begin{aligned}
		\bar{O}^{t}&=E_{\eta^{t}}\left[\Theta\left(\bar{V}^{t}-V_{th}\right)\right]_{\bar{V}^t=\bar{v}^{t,n}}\\
		&=\int_{-\infty}^{\frac{\bar{v}^{t,n}-V_{th}}{\sigma_\eta}} \frac{1}{\sqrt{2 \pi}} \exp \left(-\frac{x^2}{2}\right) d x .
	\end{aligned}
\end{equation}
and the average membrane potential $\bar{V}^{t}=\bar{v}^{t,n}$.
From the above derivation, we obtain
\begin{equation}
	\label{eq23}
	\frac{\partial \bar{o}^{t, n}}{\partial \bar{v}^{t, n}}=\frac{1}{\sqrt{2 \pi} \sigma_\eta} \exp \left(-\frac{\left(\bar{v}^{t, n}-V_{th}\right)^2}{2 \sigma_\eta{ }^2}\right),
	\end{equation}
which exactly aligns with the Gaussian surrogate function in Ref.~\cite{ref15}. Similar to this derivation, when uniform noise $\eta \sim \mathcal{U}[-a, a]$ is injected into Eqs. (\ref{eq13}) and (\ref{eq14}), differentiating the mean-field output yields the rectangular surrogate function \cite{ref15}
\begin{equation*}
	\frac{\partial \bar{o}^{t, n}}{\partial \bar{v}^{t, n}}=\frac{1}{2a}, \quad |\bar{v}^{t,n} - V_{th}| < a.
\end{equation*}
We refer to the NGN model, together with this mean-field dynamics gradient estimation collectively, as the NGN method. With all the preparation in mind, a pseudocode for the NGN method can be summarized as Algorithm \ref{alg1}.
\begin{algorithmhere}[h]
	\caption{Training of SNN based on the NGN method}\label{alg1}
	\begin{algorithmic}
		\STATE \textbf{Input:} Network input and label $(X, Y)$, timestep $T$, group size $K$, neuronal hyperparameters $\{V_m, \tau_m, \Delta t, R_m, \sigma_o\}$, and learning rate $\gamma$ 
		\STATE \textbf{Output:} The trained parameter set $W$, and network output $s$ 
		\STATE \textbf{Forward Pass:}  
		\FOR{$t = 1, 2, \dots, T$}  
		\STATE $o^{t,1} \gets \text{Encoding}(X)$  
		\FOR{the $n$-th layer with $n = 2, \dots, N$}  
		\STATE $\bar{v}^{t,n} \gets \text{GroupUpdate}(\bar{v}^{t-1,n}, \bar{o}^{t-1,n}, I^{t,n})$ //Eq.(\ref{eq13})
		\FOR{$k = 1, \dots, K$}  
		\STATE $\eta_{(k)}^n \gets \text{RandomGenerator}$  
		\STATE $(v_k^{t,n}, o_k^{t,n}) \gets \text{MemberUpdate}(\bar{v}^{t,n}, \eta_{(k)}^n)$  //Eq.(\ref{eq14})
		\ENDFOR  
		\STATE $\bar{o}^{t,n} \gets \text{SpikingAggregation}(o_k^{t,n},K)$  
		\ENDFOR  
		\ENDFOR  
		\STATE $\bar{o}^{t,N} \gets \text{DecodingLayer}(o^{t,N-1})$  
		\STATE $s \gets \text{TimeAverage}(\bar{o}^{t,N})$  
		\STATE $\mathcal{L} \gets \mathcal{L}(s, y)$ 
		\STATE \textbf{Backward Propagation:}  
		\STATE Calculate $\frac{\partial \mathcal{L}}{\partial s}$  
		\FOR{the $n$-th layer with $n = N, N-1, \dots, 1$}  
		\STATE $\frac{\partial \bar{o}^{t,n}}{\partial \bar{v}^{t,n}} \gets \text{CollectiveGrad}(\bar{v}^{t,n}, \xi)$ //Eq.(\ref{eq23})
		\STATE $\frac{\partial \mathcal{L}}{\partial W^n} \gets \text{AutoGrad}(\frac{\partial \bar{o}^{t,n}}{\partial \bar{v}^{t,n}})$  
		\STATE $W^n \gets W^n - \gamma \frac{\partial \mathcal{L}}{\partial W^n}$ 
		\ENDFOR 
	\end{algorithmic}
\end{algorithmhere}
\subsection{Mechanism for mitigating gradient mismatching}

The gradient mismatching issue is characterized by two factors: (1) the divergence between the Dirac-delta function and surrogate gradients, and (2) the incompatibility between discrete binary spikes and continuous differentiable signals (fictitious spikes) produced by surrogate gradients. To quantify this phenomenon, the gray shaded region in Fig.~\ref{fig4} illustrates the level of gradient mismatching in LIF-based SNNs under Gaussian surrogate gradients. In this subsection, we demonstrate how the NGN method effectively reduces gradient mismatching compared to conventional LIF neurons.

In this work, the fictitious spike from the surrogate gradient is defined as the integral of the surrogate gradient function. Ideally, if both forward and backward passes utilized mean-field dynamics of NGN models and surrogate gradients, gradient mismatching would be eliminated. However, in practice, information propagates through SNNs as discrete spike trains generated by the NGN method with finite $K$ (as illustrated in Fig.~\ref{fig4}), whereas surrogate gradients inherently operate on continuous signals. This fundamental discrepancy inevitably introduces a degree of gradient mismatching. \par
\begin{figurehere}
	\centering
	\includegraphics[width=3.15in]{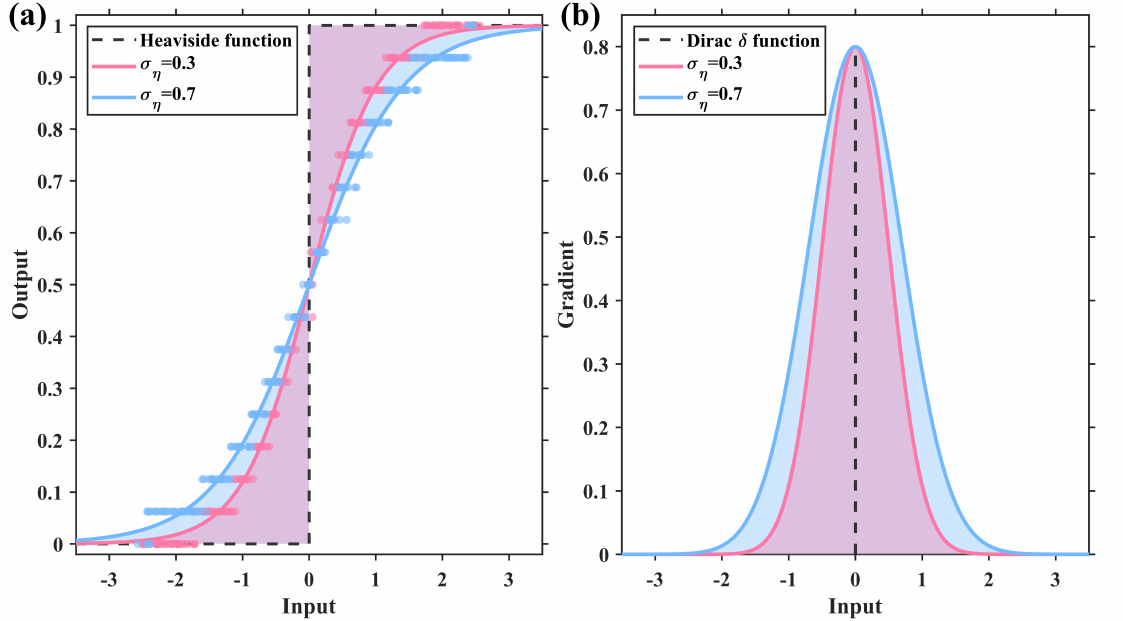}
	\caption{Comparison of the LIF method and the NGN method ($K=16$) in the forward pass (a) and the surrogate function form (b) under different noise intensities $\sigma_\eta$. The complete matching spike is equal to the integral of the surrogate gradient (the solid curve in (a)), whereas the actual graded spikes of NGN methods are represented by scatter points. The shaded area represents the level of gradient mismatching in the vanilla LIF-based SNN. \(\sigma_\eta\) is the scale parameter of surrogate gradients.}
	\label{fig4}
\end{figurehere}

In the LIF-based SNNs, the level of gradient mismatching can be measured by the area of the shaded region \cite{ref18}, as shown in Fig.~\ref{fig4}, but in the NGN method, a different metric is necessary due to the inherent randomness. Note that the gradient mismatching occurs mainly when the membrane potential $v$ is near $V_{th}$, thus for a given interval $VI=\left[V_{th}-L(\varepsilon), V_{th}+L(\varepsilon)\right]$ with $\varepsilon \in\left(0, \frac{1}{2}\right)$, we can define the level of gradient mismatching as
\begin{equation}\label{mis_prob}
	D_{\varepsilon}\left(o, o_0\right)=P\left(\left|o-o_0\right| \geq \varepsilon, v \in V I\right),
\end{equation}
where $o$ denotes the real spike from different neurons generated at $v$ when the threshold is crossed, and $o_0$ is the fictitious spike from surrogate gradient at the same $v$. That is, $D_{\varepsilon}\left(o, o_0\right)$ characterizes the probability that the discrepancy between the real spike and the fictitious spike exceeds $\varepsilon>0$ when $v \in V I$. The smaller the probability, the lower the level of gradient mismatching.
\begin{theorem}
	\label{the2}
	For the Gaussian surrogate function and its corresponding matching spike $o_0$, given any $\varepsilon$, there exists the membrane potential interval $VI$ and sufficiently large $K$ such that the following relationship holds:
	\begin{equation}
		D_{\varepsilon}(o_{NGN}, o_0) < D_{\varepsilon}(o_{LIF}, o_0).
	\end{equation}
\end{theorem}
\begin{proof}
\end{proof}

Figure~\ref{fig5}(a--c) empirically evaluates Theorem~\ref{the2} using recorded activities from a single-layer CIFAR-10 SNN. NGN with $K\geq2$ yields lower mismatching probability than LIF for $\varepsilon=0.1,0.25,0.4$, and the gap increases with $K$. The ResNet-18 loss curves in Fig.~\ref{fig5}(d--e) likewise show faster optimization and lower final loss on CIFAR-10 and CIFAR-100. In contrast, $K=1$ improves only the small-$\varepsilon$ regime and performs worse for large $\varepsilon$, indicating that effective mitigation requires a nontrivial group.
\begin{figure*}[t]
	\centering
	\includegraphics[width=0.78\textwidth]{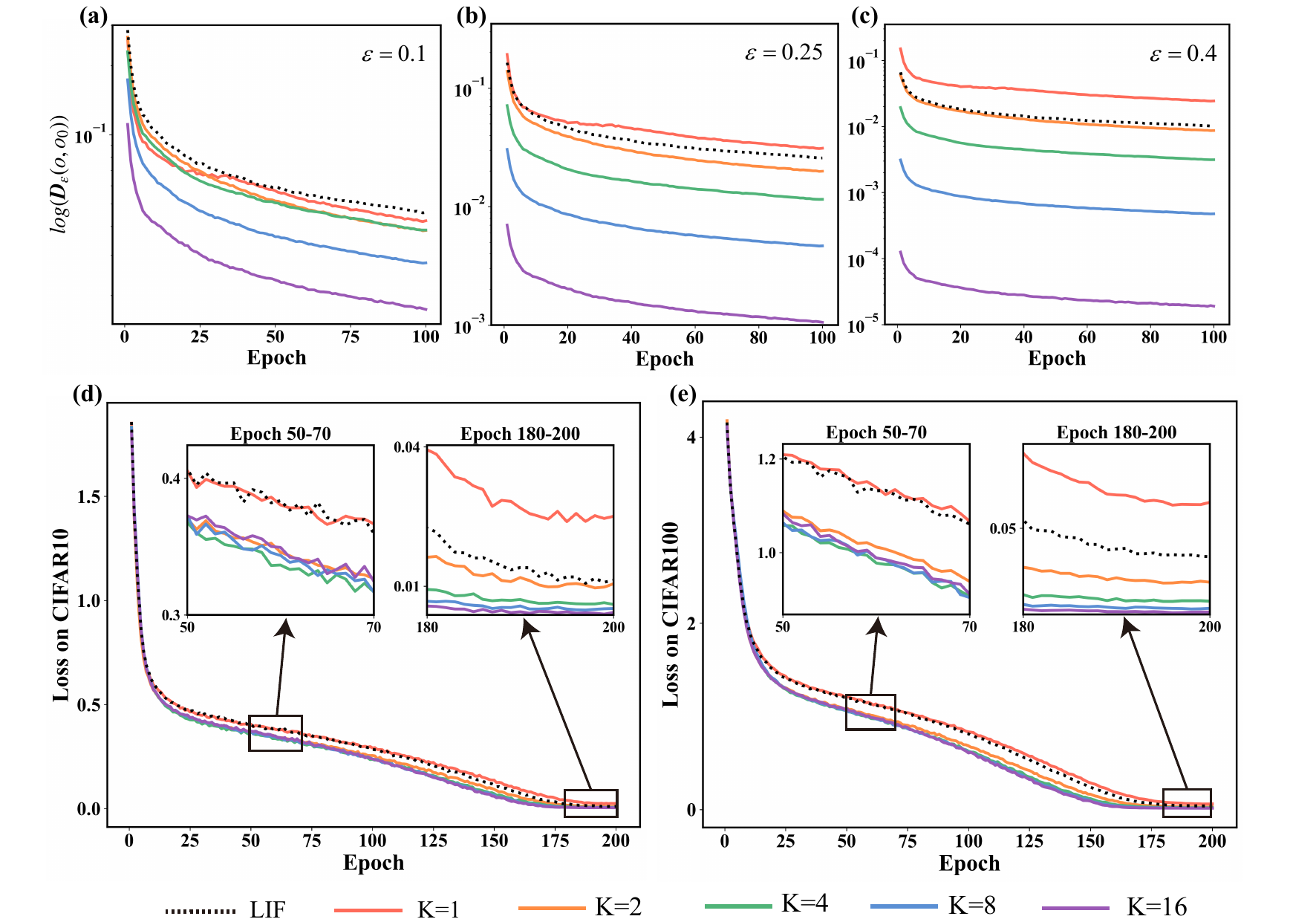}
	\caption{Comparison of gradient mismatching levels and training loss between the baseline LIF method and the NGN method with different sizes $K$. (a-c) The logarithm of the gradient mismatching probability $\log D_{\varepsilon}(o, o_0)$ calculated according to Eq.~(\ref{mis_prob}) for $\varepsilon = 0.1$, $0.25$, and $0.4$, respectively. (d-e) Training loss curves on CIFAR-10 and CIFAR-100 datasets.}
	\label{fig5}
\end{figure*}
\section{Experiments}

We evaluate NGN on static and neuromorphic image datasets, followed by neuron-model, synchronous-resetting, group-size, and computational-cost analyses.
\subsection{Static datasets}
We use ResNet-19 on CIFAR-10/100~\cite{ref45} and VGG-13 on Tiny-ImageNet~\cite{ref46}, with direct encoding~\cite{ref20}. Table~\ref{tab1} shows that NGN remains effective at low latency. Without advanced augmentation, it obtains $94.98\%/95.50\%$ on CIFAR-10 and $78.81\%/79.35\%$ on CIFAR-100 at $T=2/4$. Auto-augmentation and Cutout further raise these results to $96.94\%/97.02\%$ and $82.61\%/83.35\%$, respectively. On Tiny-ImageNet, NGN reaches $63.18\%$ at $T=4$, while combining NGN with TET produces the best reported results in the table: $97.26\%$, $83.88\%$, and $64.70\%$ on the three static benchmarks. These comparisons show that NGN benefits from temporal objectives and data augmentation without relying on long simulation windows.

On CIFAR-10, the NGN-based ResNet-19 improves from $94.98\%$ at $T=2$ to $95.50\%$ at $T=4$ without advanced augmentation. With augmentation, the corresponding results increase to $96.94\%$ and $97.02\%$. A similar pattern occurs on CIFAR-100, where augmentation raises accuracy from $78.81\%$ to $82.61\%$ at $T=2$ and from $79.35\%$ to $83.35\%$ at $T=4$. This consistent improvement indicates that the NGN representation remains compatible with conventional image-level regularization.

On Tiny-ImageNet, NGN reaches $62.21\%$ and $62.59\%$ at $T=2$ and $T=4$ without Cutout, and $62.87\%$ and $63.18\%$ with Cutout. Incorporating TET further improves the three static benchmarks, yielding $97.26\%$ on CIFAR-10, $83.88\%$ on CIFAR-100, and $64.70\%$ on Tiny-ImageNet. Thus, the group-neuron mechanism and temporal loss provide complementary improvements rather than redundant effects.
\begin{table*}[t]  
	\centering  
	\caption{Comparison of the NGN method ($K=16$) with several existing methods on static image datasets, with * denoting the cooperation of Cutout and/or data auto-augmentation techniques.}  
	\label{tab1}
	\small
	\renewcommand{\arraystretch}{1.2}
		\begin{tabular}{ccccc}  
			\hline  
			\textbf{Dataset} & \textbf{Method} & \textbf{Architecture} & \textbf{Time steps} & \textbf{Top1-Accuracy} \\   
			\hline  		
			\multirow{10}{*}{\centering CIFAR-10} 
			& STBP-tdBN \cite{ref40} & ResNet-19 & 6 & 93.16 \\   
			& TET \cite{ref43} & ResNet-19 & 6 & 94.50 \\   
			& IM-LOSS \cite{ref19} & ResNet-19 & 2 & 93.85 \\   
			& NSNN \cite{ref50} & CIFAR Net & 4 & 94.30 \\   
			& Diet-SNN \cite{ref20} & VGG-16 & 4 & 92.76 \\   
			& CLIF+TET \cite{ref51} & ResNet-18 & 8 & 96.69* \\   
			& SML \cite{ref47} & ResNet-19 & 4 & 96.82* \\   
			& TTS \cite{ref52} & ResNet-19 & 2 & 95.80* \\   
			\cline{2-5}
			& \multirow{2}{*}{\textbf{NGN}} & \multirow{2}{*}{\textbf{ResNet-19}} & \textbf{2} & \textbf{94.98$\pm$0.08/ 96.94$\pm$0.08*} \\   
			&  &  & \textbf{4} & \textbf{95.50$\pm$0.10/ 97.02$\pm$0.09*} \\   
			& \textbf{NGN+TET} & \textbf{ResNet-19} & \textbf{4} & \textbf{95.89$\pm$0.08 / 97.26$\pm$0.06*} \\ 
			\hline  
			\multirow{10}{*}{\centering CIFAR-100}   
			& TET \cite{ref43} & ResNet-19 & 6 & 74.72 \\   
			& IM-LOSS \cite{ref19} & VGG-16 & 5 & 70.18 \\   
			& CLIF+TET \cite{ref51} & ResNet-18 & 8 & 80.89* \\   
			& SML \cite{ref47} & ResNet-19 & 4 & 79.18 / 81.70* \\   
			& Diet-SNN \cite{ref20} & VGG-16 & 5 & 69.27 \\   
			& NSNN \cite{ref50} & CIFAR Net & 4 & 74.17* \\   
			\cline{2-5} 
			& \multirow{2}{*}{\textbf{NGN}} & \multirow{2}{*}{\textbf{ResNet-19}} & \textbf{2} & \textbf{78.81$\pm$0.10 / 82.61$\pm$0.10*} \\   
			&  &  & \textbf{4} & \textbf{79.35$\pm$0.14 / 83.35$\pm$0.16*} \\   
			& \textbf{NGN+TET} & \textbf{ResNet-19} & \textbf{4} & \textbf{81.05$\pm$0.19 / 83.88$\pm$0.09*} \\
			\hline  
			\multirow{7}{*}{\centering Tiny-ImageNet}   
			& Spike-thrift \cite{ref53} & VGG-16 & 150 & 51.92 \\   
			& ASGL \cite{ref18} & VGG-13 & 8 & 56.81* \\   
			& CLIF+TET \cite{ref51} & VGG-13 & 4 & 63.16* \\   
			\cline{2-5} 
			& \multirow{2}{*}{\textbf{NGN}} & \multirow{2}{*}{\textbf{VGG-13}} & \textbf{2} & \textbf{62.21$\pm$0.25 / 62.87$\pm$0.17*} \\   
			&  &  & \textbf{4} & \textbf{62.59$\pm$0.19 / 63.18$\pm$0.17*} \\   
			& \textbf{NGN+TET} & \textbf{VGG-13} & \textbf{2} & \textbf{62.61$\pm$0.24 / 63.90$\pm$0.23*} \\   
			&  &  & \textbf{4} & \textbf{63.69$\pm$0.12 / 64.70$\pm$0.13*} \\      
			\hline   
		\end{tabular} 
\end{table*} 
\newcommand{\neurotable}{%
\begin{table*}[t]  
	\caption{Comparison of the NGN method with TET loss implementation with the existing methods on neuromorphic image datasets, with † denoting the implemented results by us.}  
	\label{tab2}  
	\begin{center}  
		\small
		\renewcommand{\arraystretch}{0.9}
			\begin{tabular}{ccccc}  
				\hline  
				\textbf{Dataset} & \textbf{Method} & \textbf{Architecture} & \textbf{Time steps} & \textbf{Top1-Accuracy} \\   
				\hline  
				\multirow{5}{*}{\centering DVS-Gesture}   
				& NSNN+TET \cite{ref50} & 7B-Net & 16 & 96.88 \\   
				& SSNN \cite{ref57} & VGG-9 & 8 & 94.91 \\   
				& TET † \cite{ref43} & SNN-5 & 10 & 95.46 \\   
				& DSGM+DTAM \cite{ref21} & SNN-5 & 10 & 96.69 \\   
				& \textbf{NGN+TET} & \textbf{SNN-5} & \textbf{10} & \textbf{97.88$\pm$0.29} \\   
				\hline  
				\multirow{5}{*}{\centering N-Caltech101}   
				& SSNN \cite{ref57} & VGG-9 & 8 & 79.25 \\   
				& tdBN+NDA \cite{ref58} & VGG-11 & 10 & 78.20 \\   
				& TET † \cite{ref43} & VGG-SNN & 10 & 81.68 \\   
				& DSGM+DTAM \cite{ref21} & VGG-SNN & 10 & 76.39 \\   
				& \textbf{NGN+TET} & \textbf{VGG-SNN} & \textbf{10} & \textbf{84.04$\pm$0.34} \\   
				\hline  
				\multirow{8}{*}{\centering CIFAR10-DVS}   
				& STBP-tdBN \cite{ref40} & ResNet-19 & 10 & 67.8 \\   
				& TET \cite{ref43} & VGG-SNN & 10 & 83.17 \\   
				& IM-Loss \cite{ref19} & ResNet-19 & 10 & 72.60 \\   
				& NSNN+TET \cite{ref50} & VGG-SNN & 10 & 79.52 \\   
				& CLIF+TET \cite{ref51} & VGG-SNN & 10 & 86.1 \\   
				& SML \cite{ref47} & VGG-SNN & 10 & 85.23 \\   
				& TKS \cite{ref60} & ResNet-19 & 10 & 85.3 \\   
				& \textbf{NGN+TET} & \textbf{VGG-SNN} & \textbf{10} & \textbf{87.35$\pm$0.33} \\   
				\hline  
			\end{tabular}  
	\end{center}  
\end{table*}  
}
\subsection{Neuromorphic datasets}
We evaluate DVS-Gesture, N-Caltech101, and CIFAR10-DVS~\cite{ref54,ref55,ref4} using SNN-5 or VGG-SNN. As shown in Table~\ref{tab2}, NGN+TET obtains $97.88\%$, $84.04\%$, and $87.35\%$, respectively, at $T=10$. The gains are consistent across sensor types: on DVS-Gesture it exceeds NSNN+TET by 1.00 percentage point; on N-Caltech101 it improves over TET by 2.36 points; and on CIFAR10-DVS it exceeds TET, TKS, and CLIF+TET. This consistency indicates that the group response complements temporal training objectives rather than depending on one particular event representation.

For DVS-Gesture, NGN+TET exceeds NSNN+TET ($96.88\%$) and the reproduced TET baseline ($95.46\%$). On N-Caltech101, it outperforms SSNN ($79.25\%$), tdBN+NDA ($78.20\%$), and the reproduced TET result ($81.68\%$). These improvements are obtained using the same ten-step observation window.

CIFAR10-DVS provides the most challenging event-based comparison. NGN+TET reaches $87.35\%$, compared with $83.17\%$ for TET, $85.30\%$ for TKS, $85.23\%$ for SML, and $86.10\%$ for CLIF+TET. The results across all three datasets support the use of NGN with different event representations and network backbones.

\neurotable

\subsection{Performance comparison across neuron models}
To evaluate the contribution of NGN, we compare it with alternative neuron models under the same network architecture, loss function, optimizer, learning rate, and initialized weights. 
\gdef\modeltable{%
\begin{tablehere}  
	\vspace*{-10pt}
	\begin{minipage}{\columnwidth}
	\caption{Test accuracy and loss of neuron models on CIFAR datasets using ResNet-18 with TET ($T=4$, $K=8$).}
	\label{tab3}
	\centering
		{\small
		\setlength{\tabcolsep}{9pt}
		\renewcommand{\arraystretch}{1.05}
			\begin{tabular}{@{}lrr@{}}  
				\hline
				\textbf{Neuron} & \multicolumn{1}{c}{\textbf{Accuracy (\%)}} & \multicolumn{1}{c}{\textbf{Loss}} \\   
				\hline
				\multicolumn{3}{c}{\textbf{CIFAR-10}} \\   
				\hline
				Vanilla LIF & 93.94 & 0.3853 \\   
				CLIF & 94.00 & 0.3807 \\   
				Vanilla NLIF & 93.69 & 0.3981 \\   
				Bernoulli NLIF & 92.87 & 0.4056 \\   
				\textbf{NGN} & \textbf{95.17} & \textbf{0.3385} \\   
				\hline
				\multicolumn{3}{c}{\textbf{CIFAR-100}} \\   
				\hline
				Vanilla LIF & 76.54 & 1.18 \\   
				CLIF & 77.07 & 1.20 \\   
				Vanilla NLIF & 74.69 & 1.22 \\   
				Bernoulli NLIF & 73.70 & 1.21 \\   
				\textbf{NGN} & \textbf{78.06} & \textbf{1.09} \\   
				\hline
			\end{tabular}
		}
	\end{minipage}
\end{tablehere}  
}

Table~\ref{tab3} compares deterministic LIF, CLIF, two noisy LIF variants, and NGN under the same ResNet-18/TET setup. Direct Gaussian or Bernoulli noise injection reduces accuracy relative to LIF, showing that noise alone is insufficient. In contrast, NGN reaches $95.17\%$ on CIFAR-10 and $78.06\%$ on CIFAR-100, improving over LIF by 1.23 and 1.52 points.

\modeltable

The two noisy single-neuron baselines are particularly informative. Vanilla NLIF obtains $93.69\%$ and $74.69\%$, while Bernoulli NLIF obtains $92.87\%$ and $73.70\%$ on CIFAR-10 and CIFAR-100, respectively. Both are below deterministic LIF, confirming that the improvement of NGN cannot be attributed to noise injection alone. Population averaging and the matched learning formulation are required to convert noise into a useful computational mechanism.

\subsection{Necessity of synchronous resetting}
Table~\ref{tab_sync_reset} isolates synchronous resetting on CIFAR-10 using ResNet-18, $T=4$, and matched optimization settings. The no-sharing variant retains eight independent membrane histories; the complete NGN uses the shared state in Eqs.~(\ref{eq13})--(\ref{eq14}).

\begin{table*}[ht]
\centering
\small
\setlength{\tabcolsep}{5pt}
\caption{Component ablation on CIFAR-10. Peak memory and time per epoch characterize training cost, and Accuracy reports the best test accuracy over 200 epochs.}
\label{tab_sync_reset}
\begin{tabular}{lcccc}
\hline
Variant & $K$ & Memory (MB) & Time/epoch (s) & Accuracy (\%)\\
\hline
LIF (Gaussian surrogate) & 1 & 2235.8 & 79.2 & 94.23\\
Noisy LIF & 1 & 2237.4 & 81.8 & 93.32\\
NGN with synchronous resetting & 8 & 2253.7 & 129.0 & 95.16\\
NGN w/o shared resetting & 8 & 8933.8 & 791.5 & 95.10\\
\hline
\end{tabular}
\end{table*}

Without shared resetting, peak memory and epoch time increase by 3.96$\times$ and 6.14$\times$, while accuracy changes only from $95.16\%$ to $95.10\%$. Synchronous resetting therefore makes finite-population coding practical by compressing independent temporal histories into one recurrent state.

\subsection{Accuracy--efficiency trade-offs across group sizes} 
To examine how group size affects performance, we vary only the group size under otherwise identical ResNet-18 and optimization settings. We denote the group sizes in the training and testing phases as $K_{\text{train}}$ and $K_{\text{test}}$, respectively. Here, $K_{\text{train}}$ governs learning and gradient propagation, while $K_{\text{test}}$ influences prediction through ensemble-based spatial information aggregation.

\begin{figurehere}
	\centering
	\includegraphics[width=2.50in]{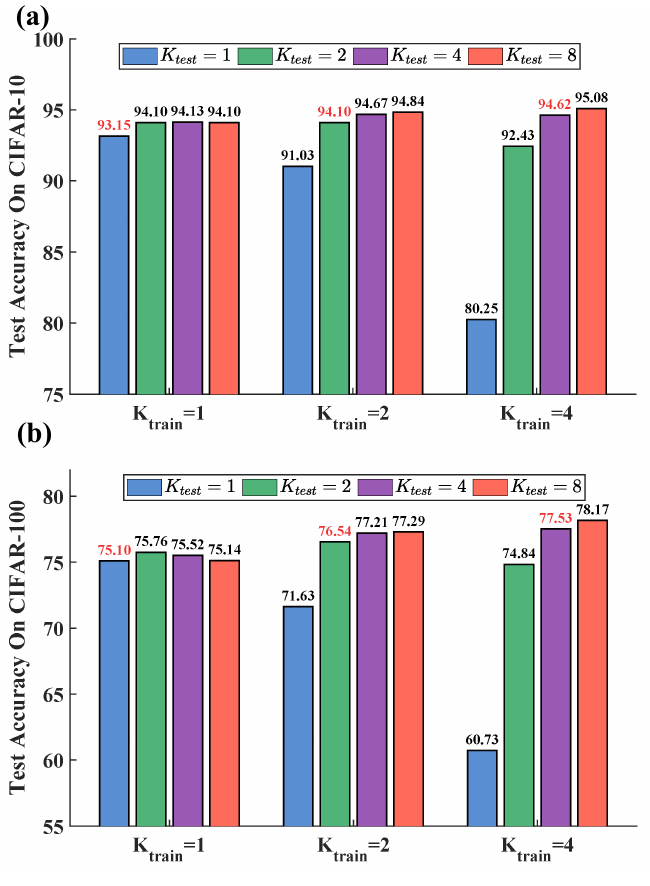}
	\caption{Accuracy under different $K_{\text{train}}$ and $K_{\text{test}}$ on (a) CIFAR-10 and (b) CIFAR-100; red values indicate equal group sizes.}
	\label{fig8}
\end{figurehere}
\vspace*{-2pt}

Figure~\ref{fig8} displays the test accuracy on CIFAR-10 and CIFAR-100 under different combinations of $K_{\text{train}} \in \{1,2,4\}$ and $K_{\text{test}} \in \{1,2,4,8\}$. Three observations can be made. First, when $K_{\text{test}} = K_{\text{train}}$, the NGN method achieves stable accuracy on both datasets. Second, when $K_{\text{test}} > K_{\text{train}}$, the accuracy generally increases. For instance, when $K_{\text{train}}=2$, increasing $K_{\text{test}}$ from 2 to 8 improves the accuracy. Third, when $K_{\text{test}} < K_{\text{train}}$, the accuracy decreases, especially when $K_{\text{test}} < K_{\text{train}}=4$.

These results show that the group size used during inference affects the performance of NGN. A larger $K_{\text{test}}$ can improve accuracy when the model is trained with a smaller group, whereas reducing $K_{\text{test}}$ below $K_{\text{train}}$ may cause an accuracy drop. Independently setting $K_{\text{train}}$ and $K_{\text{test}}$ therefore provides a trade-off between training efficiency and inference accuracy.
\subsection{Energy and memory cost of NGN}
The memory complexity of SNNs during training can be approximated as $O(W + T \cdot (X + H))$~\cite{PSN2023}, where $W$ represents the number of synapses, $T$ denotes the number of time steps, $X$ is the input/output of all layers at a single time step, and $H$ indicates the hidden state of all layers at a single time step. For LIF neurons, the hidden state includes the membrane potential before and after resetting.

For NGN methods, the hidden state is consistent with that of general neurons but requires generating $K \cdot X$ random numbers through pseudo-random number generators. Therefore, the additional memory cost of NGN is $O(T \cdot (K-1) \cdot X)$ compared with LIF neurons. However, the input/output of neurons occupies only a small portion of the overall network memory, as memory is predominantly consumed by synaptic weights $W$. As shown in Table~\ref{tabnew}, maximum memory consumption rises only marginally with group size. For instance, on CIFAR-10, memory usage grows from 29.70 MB/image for LIF to 30.60 MB/image for NGN with $K=16$, an increase of approximately 3\%. Training time per epoch increases more noticeably with $K$, mainly because more pseudo-random numbers must be generated; on CIFAR-10, it rises from 1.13 min/epoch for LIF to 2.46 min/epoch for NGN with $K=16$.

This increase in memory and time usage is unavoidable on conventional GPU platforms. In inherently stochastic neuromorphic hardware, particularly memristor-based neuromorphic chips, the implementation may be more favorable. Memristive devices exhibit intrinsic stochasticity arising from random conductive-filament formation and rupture, ion migration, and thermal fluctuations~\cite{refM1}. Although such stochasticity has traditionally been treated as a challenge requiring mitigation~\cite{refM2}, recent work has explored its computational use~\cite{refM3}. This intrinsic randomness may provide the random-number generation required by NGN without dedicated circuitry.

The additional energy consumption of NGN mainly comes from pseudo-random-number generation and extra synaptic operations. We estimate the latter through synaptic operations (SOps). For ResNet-18, NGN with $K=8$ requires $5.66 \times 10^2$M SOps, compared with $5.58 \times 10^2$M for LIF~\cite{ref60}, corresponding to an increase of approximately 1.43\%. Thus, the measured increase in synaptic operations remains modest.

\begin{tablehere}
	\centering
	\refstepcounter{table}\label{tabnew}
	\noindent\textbf{Table \thetable.}\enskip Memory and training time comparison across LIF and NGN methods with different group sizes $K$.\par\smallskip
	\resizebox{\columnwidth}{!}{%
	\begin{tabular}{lcccc}
		\hline
		 & \multicolumn{2}{c}{\textbf{CIFAR-10}} & \multicolumn{2}{c}{\textbf{CIFAR10-DVS}} \\
		\cline{2-5}
		& \textbf{Time} & \textbf{Memory} & \textbf{Time} & \textbf{Memory} \\
		& \textbf{(min/ep)} & \textbf{(MB/img)} & \textbf{(min/ep)} & \textbf{(MB/img)} \\
		\hline
		LIF & 1.13 & 29.70 & 1.83 & 159.3 \\
		K=1 & 1.32 & 29.70 & 1.87 & 159.3 \\
		K=2 & 1.38 & 29.78 & 1.90 & 159.3 \\
		K=4 & 1.53 & 29.90 & 2.02 & 159.3 \\
		K=8 & 1.83 & 30.14 & 2.26 & 159.3 \\
		K=16 & 2.46 & 30.60 & 2.71 & 160.0 \\
		\hline
	\end{tabular}%
	}
\end{tablehere}

\section{Conclusion}

We proposed a novel NGN model to address spatiotemporal information loss and forward--backward gradient mismatching by combining finite-sample population coding, constructive noise, and population-level synchronous resetting. Based on this NGN model, we have then presented the NGN method to enhance representational capacity and improve training efficiency simultaneously. Experiments on static and neuromorphic datasets show that NGN improves classification performance over deterministic LIF baselines under the evaluated settings. Particularly, it has the following two merits.

The first merit is that the NGN method can enhance the representation and transmission ability of spatiotemporal information, as revealed by theoretical analysis and numerical verification. In fact, this point is not hard to understand, because it can be explained by the effect of parameter-induced stochastic resonance: the ``trained'' optimal parameters can maximally magnify the weak sub-threshold feature in the presence of a suitable amount of noise. The second merit of the NGN method lies in its ability to effectively manage the computational requirements associated with member neurons. Note that in the absence of synchronous resetting, the NGN model is reduced to the parallel array of noisy LIF neurons. In the CIFAR-10 ablation, removing shared resetting increased peak memory by 3.96$\times$ and epoch time by 6.14$\times$, while the two variants reached similar best accuracies after 200 epochs.

The population coding mechanism introduces additional computation, particularly for random-number generation and member-wise sampling. As discussed in Section 5.6, the measured memory and synaptic-operation increases remain modest, whereas training time grows with group size. Future work should therefore focus on more efficient stochastic implementations and hardware-aware control of the training and inference group sizes.

\section*{Acknowledgments}
This research is financially supported by the National Natural Science Foundation of China under Grant Nos. 12572040 and 12172268.

\end{document}